\documentclass[11pt]{article}

\def\f{{\bf f}}
\def\gbold{{\bf g}}

\def\h{{\bf h}}

\def\ubold{{\bf u}}
\def\vbold{{\bf v}}
\def\w{{\bf w}}
\def\x{{\bf x}}
\def\y{{\bf y}}
\def\z{{\bf z}}

\def\C{{\mathcal C}}

\def\F{{\mathcal F}}

\def\V{{\mathcal V}}

\def\R{{\mathbb R}}

\def\al{\alpha}

\def\D{\Delta}
\def\e{\epsilon}
\def\g{\gamma}

\def\om{\omega}
\def\OM{\Omega}

\def\s{\sigma}
\def\SI{\Sigma}

\def\Th{\Theta}

\def\bl{{\boldsymbol \lambda}}

\def\bth{{\boldsymbol \theta}}
\def\bphi{{\boldsymbol \phi}}

\def\bxi{{\boldsymbol \xi}}
\def\bzeta{{\boldsymbol \zeta}}

\def\tai{t \ap \infty}

\def\ap{\rightarrow}

\def\seq{\subseteq}

\def\bz{{\bf 0}}
\def\imp{\; \Longrightarrow \;}

\def\fa{\; \forall}

\def\as{\mbox{ a.s.}}

\def\nm{\Vert}

\renewcommand{\and}{\mbox{$\wedge$}}

\def\bxt{\bxi_{t+1}}
\def\bzt{\bzeta_{t+1}}

\def\gJ{\nabla J}

\newcommand{\bc}{\begin{center}}
\newcommand{\ec}{\end{center}}
\newcommand{\be}{\begin{equation}}
\newcommand{\ee}{\end{equation}}
\newcommand{\bd}{\begin{displaymath}}
\newcommand{\ed}{\end{displaymath}}
\newcommand{\ba}{\begin{array}}
\newcommand{\ea}{\end{array}}
\newcommand{\ben}{\begin{enumerate}}
\newcommand{\een}{\end{enumerate}}
\newcommand{\bit}{\begin{itemize}}
\newcommand{\eit}{\end{itemize}}
\newcommand{\beq}{\begin{eqnarray}}
\newcommand{\eeq}{\end{eqnarray}}
\newcommand{\btab}{\begin{tabular}}
\newcommand{\etab}{\end{tabular}}
\newcommand{\bfig}{\begin{figure}}
\newcommand{\efig}{\end{figure}}
\newcommand{\btp}{\begin{tikzpicture}}
\newcommand{\etp}{\end{tikzpicture}}

\newcommand{\nmeu}[1]{ \nm #1 \nm_2 }
\newcommand{\nmeusq}[1]{ \nm #1 \nm_2^2 }

\newcommand{\IP}[2]{ \langle #1 , #2 \rangle }

\def\nmsl1{\nm_{{\rm SL1}}}

\def\gJ{\nabla J}
\def\gJt{\gJ(\bth_t)}

\usepackage{graphicx}
  \DeclareGraphicsExtensions{.jpeg,.png,.jpg,.eps,.pdf}
\usepackage{tikz}
\usepackage{epstopdf}
\usepackage{hyperref}

\usepackage{amssymb,amsmath,amsthm}

\def\gJt{\gJ(\bth_t)}

{\bf}{\rm}
\newtheorem{lemma}{Lemma}{\bf}{\rm}
\newtheorem{definition}{Definition}{\bf}{\rm}
\newtheorem{theorem}{Theorem}{\bf}{\rm}
{\bf}{\rm}

\begin{document}

\title{
Convergence of Two Time-Scale Stochastic Approximation: \\
A Martingale Approach
}

\author{Mathukumalli Vidyasagar
\thanks{Indian Institute of Technology Hyderabad,
Kandi, Telangana 502285; email: m.vidyasagar@iith.ac.in.}
}

\maketitle

\begin{abstract}

In this paper, we analyze 
the two time-scale stochastic approximation (TTSSA) algorithm
introduced in \cite{Borkar97} using a martingale approach.
Until now, almost all of the analysis has been based on the ODE approach.
The martingale approach leads to simple sufficient conditions for the iterations
to be bounded almost surely, as well as estimates on the rate of convergence
of the mean-squared error of the TTSSA algorithm to zero.
Thus the boundedness of the iterations is a conclusion and not a hypothesis.

Our theory is applicable to \textit{nonlinear} equations, in
contrast to many papers in the TTSSA literature which assume that the equations
are linear.
The convergence of TTSSA is proved in the ``almost sure'' sense,
in contrast to earlier papers on TTSSA that establish convergence
in distribution, convergence in the mean, and the like.
Moreover, in this paper we establish different rates of convergence for
the fast and the slow subsystems, perhaps for the first time.
Finally, all of the above results to continue to hold in the case where
the two measurement errors have nonzero conditional
mean, and/or have conditional variances that grow without bound as the
iterations proceed.
This is in contrast to previous papers which assumed that the errors form
a martingale difference sequence with uniformly bounded conditional variance.

It is shown that when the measurement errors have zero conditional
mean and the conditional variance remains bounded, the mean-squared
error of the iterations converges to zero
at a rate of $o(t^{-\eta})$ for all $\eta \in (0,1)$.
This improves upon the rate of $O(t^{-2/3})$ proved in 
\cite{Doan-TAC23} (which is the best bound available to date).
Our bound is virtually the same as the rate of $O(t^{-1})$ proved in
\cite{Doan-arxiv24}, but for a \text{Polyak-Ruppert averaged} version of
TTSSA, and not directly.
Rates of convergence are also established for the case where the
errors have nonzero conditional mean and/or unbounded conditional variance.

\end{abstract}

\section{Introduction}\label{sec:Intro}

\subsection{Problem Statement}\label{ssec:Prob}

In this paper, we study the problem of solving \textit{coupled}
nonlinear equations of the form
\be\label{eq:11}
\f(\bth,\bphi) = \bz , \quad \gbold(\bth,\bphi) = \bz ,
\ee
where $\bth \in \R^d, \bphi \in \R^l$, and $\f : \R^d \times \R^l \ap \R^d$,
$\gbold : \R^d \times \R^l \ap \R^l$.
To achieve this objective, we have available only \textit{noise-corrupted}
measurements of $\f(\cdot,\cdot)$ and $\gbold(\cdot,\cdot)$.
Specifically, let $\y_{t+1}$ denote a noisy measurement of $\f(\bth_t,\bphi_t)$,
and let $\z_{t+1}$ denote a noisy measurement of $\gbold(\bth_t,\bphi_t)$.
Then the iterations proceed as follows:
\be\label{eq:12}
\bth_{t+1} = \bth_t + \al_t \y_{t+1} , \quad
\bphi_{t+1} = \bphi_t + \beta_t \z_{t+1}, 
\ee
where $\al_t, \beta_t$ denote the step sizes.
This approach is known in the literature as \textbf{Two Time-Scale
Stochastic Approximation (TTSSA)}.
It was first introduced in \cite{Borkar97}.

Let $\F_t$ denote the $\s$-algebra generated by $\bth_0, \bphi_0$,
and the measurement sequences $\y_1^t := ( \y_1 , \cdots , \y_t)$, 
$\z_1^t := ( \z_1 , \cdots , \z_t)$.
Note that there are no $\y_0$ and $\z_0$.
For convenience, given a random variable $X$, let $E_t(X)$ denote the
conditional expectation $E(X|\F_t)$.
To characterize the nature of the noisy measurements, define
\be\label{eq:12a}
\x_t := E_t(\y_{t+1} - \f(\bth_t, \bphi_t)) , \quad
\bxt := \y_{t+1} - \f(\bth_t, \bphi_t) - \x_t ,
\ee
and similarly, define
\be\label{eq:12b}
\w_t := E_t(\z_{t+1} - \gbold(\bth_t, \bphi_t)) , \quad
\bzt := \z_{t+1} - \gbold(\bth_t, \bphi_t) - \w_t .
\ee
From the ``tower'' property of conditional expectations 
\cite[Sec.\ 9.7, p.\ 88]{Williams91}, it follows that
\be\label{eq:12c}
E_t(\bxt) = \bz , \quad E_t(\bzt) = \bz .
\ee
Thus one can think of $\x_t$ and $\w_t$ as the ``biases'' of the two
measurements $\y_{t+1}$ and $\z_{t+1}$ respectively,
and of $\bxt$ and $\bzt$ as the ``unpredictable'' part of the errors.
Note that, as a result of the above definitions, it follows that
\be\label{eq:12d}
E_t(\nmeusq{\y_{t+1}}) = \nmeusq{\f(\bth_t, \bphi_t)) + \x_t) }
+ E_t( \nmeusq{\bxt}) , 
\ee
\be\label{eq:12e}
E_t(\nmeusq{\z_{t+1}}) = \nmeusq{\gbold(\bth_t, \bphi_t)) + \w_t)}
+ E_t( \nmeusq{\bzt}) .
\ee
With these definitions, we can write
\be\label{eq:13}
\y_{t+1} = \f(\bth_t, \bphi_t) + \x_t + \bxt , \quad
\z_{t+1} = \gbold(\bth_t, \bphi_t) + \w_t + \bzt .
\ee

The objective of Two Time-Scale Stochastic Approximation (TTSSA)
is to determine conditions under which the iterations in \eqref{eq:12}
converge to a solution of \eqref{eq:11}, and if possible, obtain
estimate of the \textit{rate} at which convergence takes place.
The objective of this paper is to present the best results thus far
on the convergence and behavior of TTSSA.

\subsection{Overview of Stochastic Approximation}\label{ssec:SA}

The Stochastic Approximation (SA) algorithm was introduced in
\cite{Robbins-Monro51}, to solve iteratively an equation of the
form\footnote{In the original paper, Robbins and Monro considered only
scalar equations, that is, $d=1$.}
\be\label{eq:14}
\f(\bth) = \bz ,
\ee
using only noisy measurements of the form $\f(\bth_t) + \bxt$.
The SA algorithm is
\be\label{eq:15}
\bth_{t+1} = \bth_t + \al_t [ \f(\bth_t) + \bxt ] ,
\ee
where $\al_t$ is the step size.
Almost at once, Kiefer and Wolfowitz \cite{Kief-Wolf-AOMS52} and
Blum \cite{Blum54} addressed the problem of finding a stationary point
of a $\C^1$ function $J: \R^d \ap \R$ when only noisy measurements
of $J(\bth_t)$ are available.
Note that if noisy measurements of the \textit{gradient} $\gJt$
are available, this would be the same as the problem sutdied in
\cite{Robbins-Monro51}.

In much of the literature, it is assumed that the measurement error $\bxt$
has zero conditional mean and a conditional covariance that is uniformly
bounded as a function of $t$, the iteration counter.
One noteworthy feature of the two early papers \cite{Kief-Wolf-AOMS52,Blum54}
is that the measurement error
has nonzero conditional mean, and also, its conditional variance
increases without bound as the iterations proceed.
Thus for the case where only function measurements are used
(the so-called ``zeroth-order methods''),
any analysis of the Stochastic Gradient Descent (SGD)
algorithm would be incomplete unless the analysis is applicable to the
case of nonzero conditional mean and unbounded conditional variance.
The most general paper that can cope with these two features is
\cite{MV-RLK-SGD-JOTA24}.
We mention this because almost all the papers on TTSSA assume that
the error has zero conditional mean and bounded conditional variance.
One of the contributions of this paper is to extend the analysis of TTSSA
to the more general situation, by adapting the proof techniques in
\cite{MV-RLK-SGD-JOTA24}.

Leaving aside the early efforts,
subsequent analysis of the behavior of the SA algorithm can be
grouped into two streams of thought.
The older approach, which also has a majority of the literature, may
be called the ODE approach.
The basis of this approach is the observation that, over time, the
trajectory of the \textit{discrete-time, stochastic} algorithm
\eqref{eq:15} ``converge'' to the \textit{deterministic} trajectories
of the associated ODE
\be\label{eq:16}
\dot{\bth} = \f(\bth) .
\ee
If the trajectories of \eqref{eq:16} are bounded almost surely,
and if the error sequence $\{ \bxt \}$ satisfies additional assumptions,
the iterations $\{ \bth_t \}$ converge to the set of equilibria of the
ODE \eqref{eq:15}.
Out of the vast literature on this approach, we mention only a few resources
in the interests of brevity.
The ODE approach was pioneered in \cite{Meerkov72a,Meerkov72b,Der-Fradkov74,Kushner-JMAA77,Ljung-TAC77b}.
Excellent book-length treatments of this approach can be found in
\cite{Kushner-Clark78,Kushner-Yin97,Kushner-Yin03,Kushner-Clark12,BMP90,Borkar08,Borkar22}.
As mentioned above, 
one of the key assumptions made in order to apply the ODE approach to a
given problem is that the SA iterations remain bounded almost surely.
The first paper where the boundedness of the iterations is a conclusion
and not an assumption is \cite{Borkar-Meyn00}.
In this paper, the authors study the ``limit'' ODE
\be\label{eq:16a}
\dot{\bth} = \f_\infty(\bth) ,
\ee
where
\bd
\f_\infty(\bth) := \lim_{r \ap \infty} \frac{\f(r \bth)}{r} .
\ed
The key assumption made in \cite{Borkar-Meyn00} is that $\bz$ is
a globally asymptotically stable equilibrium of \eqref{eq:16a}.
It is worth noting that if the function $\f(\cdot)$ exhibits
\textit{sublinear} growth, then $\f_\infty(\bth) \equiv \bz$ for all $\bth$,
so that this hypothesis can never be satisfied.
Despite this limitation, the paper \cite{Borkar-Meyn00} is widely cited
and used.

In recent years, an alternate approach has emerged, which might be called
the ``martingale approach'' \cite{MV-MCSS23}.
In this approach, it is assumed that the ODE \eqref{eq:15} not only
has some stability properties, but also a suitable ``Lyapunov function.''
Using this Lyapunov function, and applying the Robbins-Siegmund theorem
\cite{Robb-Sieg71}, it is possible to prove both the boundedness of the
iterations in \eqref{eq:12} as well as their convergence to the desired limit.
Moreover, the martingale approach is applicable even when the function
$\f(\cdot)$ exhibits sublinear growth, in contrast with the ODE approach.
However, the martingale approach does not work well when the equation
$\f(\bth) = \bz$ has multiple solutions.
Therefore each method has its own advantages and disadvantages.
A brief summary can be found towards the end of \cite[Section 3]{MV-RLK-SGD-JOTA24}.
A thorough discussion of the martingale approach to nonconvex optimization
and Stochastic Approximation will be found in
the forthcoming book \cite{MV-Book27}.

\subsection{Introduction to TTSSA}\label{ssec:TTSSA}

Now let us return to the coupled system of equations \eqref{eq:11}.
If the ODE
\bd
\left[ \ba{c} \dot{\bth} \\ \dot{\bphi} \ea \right] =
\left[ \ba{c} \f(\bth,\bphi) \\ \gbold(\bth,\bphi) \ea \right]
\ed
is globally asymptotically stable, then \eqref{eq:11} can be solved
using conventional SA applied to the map $(\f,\gbold) : \R^d \times \R^l
\ap \R^d \times \R^l$.
However, there are several applications where this stability property
does not hold.
One such example arises in Actor-Critic methods of Reinforcement Learning (RL),
in which the above ODE is not stable; however, a \textit{singularly perturbed}
version of it in the form
\be\label{eq:17}
\left[ \ba{c} \dot{\bth} \\ \e \dot{\bphi} \ea \right] =
\left[ \ba{c} \f(\bth,\bphi) \\ \gbold(\bth,\bphi) \ea \right]
\ee
is globally asymptotically stable for \textit{all sufficiently small $\e$}.
In such a case, the update equations take the form \eqref{eq:12}.
Note that the \textit{single step size} in \eqref{eq:15} is replaced by a
\textit{pair of step sizes} in \eqref{eq:12}.
For this reason, \eqref{eq:12} is referred to as Two Time-Scale Stochastic
Approximation (TTSSA).

\subsection{Literature Review}\label{ssec:Review}


Because of the vastness of the literature on this topic, the review below is
restricted only to the papers that are most relevant to the present paper.
The TTSSA algorithm was introduced in \cite{Borkar97}.
As with the ODE approach to conventional SA, the convergence of
the algorithm is established under the assumption that the iterations
remain bounded almost surely.
After the publication of \cite{Borkar97},
in \cite{Konda-Tsi-Annals04} the limit behavior of TTSSA was analyzed
in the \textit{linear} case, when both $\f$ and $\gbold$ are linear.
It is shown that a version of the Central Limit Theorem (CLT) holds,
in that the iterations converge \textit{in the distributional sense}
to the desired solution, at a rate of $O(t^{-1})$.
The assumption of linearity is relaxed in \cite{Mokk-Pell-AAP06};
however, convergence is still only in the distributional sense.
Moreover, it is assumed that the measurement errors are i.i.d.\ with zero mean.
In \cite{Tadic-ACC04,Tadic-CDC04}, the author analyzes TTSSA when
the Kushner-Clark conditions on the measurement errors are not satisfied.
In all of these papers, as with the ODE approach to 
standard SA, the convergence of
the TTSSA algorithm is proved under the assumption that the iterations
remain bounded almost surely.
In \cite{Lak-Csaba18}, the almost sure boundedness of the iterations
is established as a conclusion, and not as a hypothesis.
This paper can be thought of as an extenstion of the approach of
\cite{Borkar-Meyn00} to TTSSA.
Recent work has focused more on bounding the \textit{rate} of convergence
of the iterations.
In \cite{GSY19}, TTSSA is analyzed with constant step sizes,
and time-varying step sizes are analyzed in \cite{Doan-SICOPT21}.
In these papers, it is shown that the norm-squared of the residuals
converges to zero at the rate of $O(t^{-2/3})$.
In \cite{Doan-TAC23} once again the rate of convergence is $t^{-2/3}$.
In \cite{Doan-arxiv24}, a convergence rate of $O(t^{-1})$ is
established; however, it is for a \text{Polyak-Ruppert averaged} version of
\eqref{eq:12}, and not for \eqref{eq:12} directly.

\section{Overview of the Paper}\label{sec:Over}

\subsection{Contributions of the Paper}\label{ssec:Contrib}

The main objective of this paper
is to extend the martingale approach of \cite{MV-MCSS23}
to the two time-scale stochastic approximation (TTSSA) algorithm
in \eqref{eq:12}.
Some of the substantial outcomes of this extension are the following:
\bit
\item The theory is applicable to \textit{nonlinear} equations, in
contrast to many papers in the TTSSA literature which assume that the equations
are linear.
\item Our analysis applies to the case where the noisy measurements
have nonzero conditional mean, and/or have conditional variances that
grow without bound as the iterations proceed.
So far as we are aware, no other paper studies this case.
\item The convergence is proved in the ``almost sure'' sense,
in contrast to earlier papers on TTSSA that establish convergence
in distribution, convergence in the mean, and the like.
Since every run of an iterative algorithm results in \textit{exactly one}
sample path,
it is highly desirable that \textit{almost all} sample paths converge
to the desired limit.
We also provide estimates of the \textit{rate of convergence} to the
solution of \eqref{eq:11}.
\item Throughout it is assumed that the singularly perturbed system
\eqref{eq:17} is globally exponentially stable.
When the measurements are unbiased and the errors have bounded variance
(that is, $\x_t = \bz$, $\w_t = \bz$, and $E_t(\nmeusq{\bxt})$
and $E_t(\nmeusq{\bzt})$ are both uniformly bounded with respect to $t$),
it is shown that the mean-squared error of the
TTSSA algorithm converges to zero
at a rate of $o(t^{-\eta})$ for all $\eta \in (0,1)$.
In contrast, in \cite{Doan-TAC23} (which is the best bound available to date),
the rate of convergence is $t^{-2/3}$, which is worse than the above.
In \cite{Doan-arxiv24}, a convergence rate of $O(t^{-1})$ is
established; however, it is for a \text{Polyak-Ruppert averaged} version of
\eqref{eq:12}, and not for \eqref{eq:12} directly.
In contrast, our results are for the original TTSSA of \eqref{eq:12}.
Note that there is virtually no difference between $O(t^{-1})$ and
$o(t^{-\eta})$ for all $\eta \in (0,1)$.
Moreover, in this paper we establish different rates of convergence for
the fast and the slow subsystems, apparently for the first time.
\item Convergence rates are also established for the case where the
measurement errors are biased and/or have unbounded conditional variance
as a function of the iteration counter $t$.
Most papers do not even study this situation.
\eit

\subsection{Organization of the Paper}\label{ssec:Org}

Though the computations appear to be messy and overwhelming, the line
of the proof is quite straight-forward.
\bit
\item We begin by imposing some assumptions on the singularly perturbed
system \eqref{eq:17} ensure the existence of a suitable Lyapunov function
for its global exponential stability.
These assumptions hold more or less automatically for \textit{linear}
singularly perturbed systems, and are taken from \cite{Saberi-Khalil-TAC84}.
\item Then we add suitable assumptions on the biases $\x_t$ and $\z_t$,
and the measurement errors $\bxt$ and $\bzt$,
to ensure that the the Robbins-Siegmund theorem \cite{Robb-Sieg71} 
as generalized in \cite{MV-RLK-SGD-JOTA24} apply to
the TTSSA algorithm of \eqref{eq:12}.
\eit

\section{Preliminaries}\label{sec:Prelim}

\subsection{Convergence Theorems for Stochastic Processes}\label{ssec:Conv}

One of the most powerful techniques for proving the convergence of
stochastic processes is due to Robbins-Siegmund \cite{Robb-Sieg71},
sometimes known as the ``almost supermartingale'' convergence theorem.
It is restated as Theorem \ref{thm:Robb-Sieg} below.
The estimate of the rate of convergence makes use of a theorem

The estimate of the rate of convergence makes use of a theorem
proved in \cite{MV-RLK-SGD-JOTA24}.
In this subsection, we restate these theorems for the convenience of
the reader.
Throughout, all random variables are assumed to be defined on some
probability space $(\OM,\SI,P)$, and $\{ \F_t \}$ is assumed to be
a filtration.
Thus $\F_t \seq \F_{t+1} \seq \cdots \seq \F$.
Given a random variable $X$, we denote the conditional expectation
$E(X|\F_t)$ by $E_t(X)$.
For background material on conditional expectations and their properties,
the reader is referred to \cite{Williams91,Durrett19}.

\begin{theorem}\label{thm:Robb-Sieg}
(Robbins-Siegmund Theorem \cite{Robb-Sieg71})
Suppose $\{ z_t \} , \{ f_t \} , \{ g_t \} , \{ h_t \}$ are
stochastic processes taking values in $[0,\infty)$, adapted to some
filtration $\{ \F_t \}$, satisfying
\be\label{eq:311}
E_t( z_{t+1} ) \leq (1 + f_t) z_t + g_t - h_t \as, \fa t ,
\ee
where, as before, $E_t(z_{t+1})$ is a shorthand for $E(z_{t+1} | \F_t )$.
Then, on the set
\bd
\OM_0 := \{ \om \in \OM : \sum_{t=0}^\infty f_t(\om) < \infty \}
\cap \{ \om : \sum_{t=0}^\infty g_t(\om) < \infty \} ,
\ed
we have that $\lim_{\tai} z_t$ exists, and in addition,
$\sum_{t=0}^\infty h_t(\om) < \infty$.
In particular, if $P(\OM_0) = 1$, then $\{ z_t \}$ is bounded
almost surely and converges almost surely to some random variable $X$,
and in addition, $\sum_{t=0}^\infty h_t(\om) < \infty$ almost surely.
\end{theorem}

Theorem \ref{thm:Robb-Sieg} gives a sufficient condition for convergence,
but does not say anything about the \textit{rate} of convergence.
The concept of almost sure convergence does not lend itself naturally
to the notion of a ``rate.''
For this purpose, we introduce the following definition, inspired by
\cite{Liu-Yuan-arxiv22}:

\begin{definition}\label{def:order}
Suppose $\{ Y_t \}$ is a stochastic process, and $\{ f_t \}$
is a sequence of positive numbers.
We say that
\ben
\item $Y_t = O(f_t)$ if $\{ Y_t / f_t \}$ is bounded almost surely.
\item $Y_t = \OM(f_t)$ if $Y_t$ is positive almost surely, and
$\{ f_t/Y_t \}$ is bounded almost surely.
\item $Y_t = \Th(f_t)$ if $Y_t$ is both $O(f_t)$ and $\OM(f_t)$.
\item $Y_t = o(f_t)$ if $Y_t /f_t \ap 0$ almost surely as $\tai$.
\een
\end{definition}

The next theorem builds on Theorem \ref{thm:Robb-Sieg} and provides bounds
on the rate of convergence.
Note that all statements in the theorem hold ``almost surely,'' and
that the modifier is implicitly understood.

\begin{theorem}\label{thm:32}
(See \cite[Theorem 5.2]{MV-RLK-SGD-JOTA24}.)
Suppose $\{ z_t \} , \{ f_t \} , \{ g_t \} , \{ \al_t \}$ are
stochastic processes defined on some probability space $(\OM,\SI,P)$,
taking values in $[0,\infty)$, adapted to some
filtration $\{ \F_t \}$.
Suppose further that
\be\label{eq:315}
E_t(z_{t+1} ) \leq (1 + f_t) z_t + g_t - \al_t z_t , \fa t ,
\ee
where
\bd
\sum_{t=0}^\infty f_t(\om) < \infty , \quad
\sum_{t=0}^\infty g_t(\om) < \infty , \quad
\sum_{t=0}^\infty \al_t(\om) = \infty .
\ed
Then $z_t = o(t^{-\eta})$ for every $\eta \in (0,1)$ such that
(i) there exists an integer $T \geq 1$ such that
\be\label{eq:316}
\al_t(\om) - \eta t^{-1} \geq 0 , \fa t \geq T ,
\ee
and in addition (ii)
\be\label{eq:317}
\sum_{t=1}^\infty (t+1)^\eta g_t(\om) < \infty , \quad
\sum_{t=1}^\infty [ \al_t(\om) - \eta t^{-1} ] = \infty .
\ee
for some sufficiently large $T$.
\end{theorem}

\subsection{Assumptions on the Singularly Perturbed System}\label{ssec:Ass}

In this subsection, we state the various assumptions made on the singularly
perturbed system \eqref{eq:17}.
At a first glance, it might appear that there are quite a few assumptions.
However, \textit{all} of these assumptions hold naturally for
\textit{linear} singularly perturbed systems of the form

\be\label{eq:321}
\left[ \ba{c} \dot{\bth} \\ \e \dot{\bphi} \ea \right] =
\left[ \ba{cc} A_{11} & A_{12} \\ A_{21} & A_{22} \ea \right]
\left[ \ba{c} \bth \\ \bphi \ea \right] .
\ee
The standard assumptions for this problem are that the matrices
$A_{22}$ and $A_r := A_{11} - A_{12} A_{22}^{-1} A_{21}$ are both
Hurwitz, i.e., all of their eigenvalues have negative real parts.
With these assumptions, standard results guarantee that \eqref{eq:321}
is globally exponentially stable for \textit{all sufficiently small} $\e$.
This in turn guarantees the existence of suitable quadratic Lyapunov
functions.

Though there are a great many assumptions below, they are natural
``nonlinear'' generalizations of the above simple assumptions; they are
taken from \cite{Saberi-Khalil-TAC84}.
We use the suffix and/or subscript $F$ to denote the
``fast'' subsystem, and $S$ to denote the ``slow'' subsystem.
Note that $\bphi(t)$ is the ``fast'' variable while $\bth(t)$
is the ``slow'' variable.

Throughout, we use the abbreviation $GLC$ to denote the set of 
globally Lipschitz functions, with the domain and range being evident
from the context.
Further, if $\h \in GLC$, then $L_\h$ denotes its global Lipschitz constant.

The first set of assumptions have to do with the existence of equilibria.
\ben
\item[(F-G)] Both $\f$ and $\gbold$ belong to $GLC$
(with Lipschitz constants $L_\f$ and $L_\gbold$ respectively).
\item[(A1)] There exists a $GLC$ function $\bl : \R^d \ap \R^l$ such that 
\be\label{eq:322}
\gbold(\bth,\bl(\bth)) = \bz, \fa \bth \in \R^d ,
\ee
so that, for each $\bth \in \R^d$, the vector $\bphi^* = \bl(\bth)$
is an equilibrium of the ``fast'' ODE
\bd
\e \dot{\bphi} = \gbold(\bth,\bphi) .
\ed
\item[(A2)] We have that
\be\label{eq:323}
\f(\bz,\bl(\bz)) = \bz ,
\ee
so that $\bz$ is an equilibrium of the ``slow'' ODE
\bd
\dot{\bth} = \f(\bth,\bl(\bth)) .
\ed
\een

The next set of assumptions have to do with the existence of suitable
Lyapunov functions for both the slow and the fast systems.
\ben
\item[(VS1)] There exists a $\C^1$ function $V_S : \R^d \ap \R_+$ 
and positive constants $a_S, b_S, c_S$ such that
\be\label{eq:324}
a_S \nmeusq{\bth} \leq V_S(\bth) \leq b_S \nmeusq{\bth} , \fa \bth \in \R^d ,
\ee
\be\label{eq:325}
\IP{ \nabla V_S(\bth)}{\f(\bth,\bl(\bth))} \leq - c_S \nmeusq{\bth},
\fa \bth \in \R^d .
\ee
\item[(VS2)] $\nabla V_S(\cdot)$ belongs to $GLC$, and 
$\nabla V_S(\bz) = \bz$.
\item[(VF1)] There exists a $\C^1$ function $V_F : \R^d \times \R^l \ap \R_+$
and constants $a_F, b_F, c_F$ such that
\be\label{eq:326}
a_F \nmeusq{ \bphi - \bl(\bth)} \leq V_F(\bth , \bphi) \leq
b_F \nmeusq{ \bphi - \bl(\bth)} , \fa \bth \in \R^d , \bphi \in \R^l ,
\ee
\be\label{eq:327}
\IP{ \nabla_\bphi V_F(\bth,\bphi) }{ \gbold(\bth,\bphi) }
\leq - c_F \nmeusq{ \bphi - \bl(\bth)} , \fa \bth \in \R^d , \bphi \in \R^l .
\ee
\item[(VF2)] $\nabla V_F(\cdot,\cdot) $ is $GLC$ with respect to both arguments.
Further, (i) $\nabla_{(\bth,\bphi)} V_F(\bz,\bz) = \bz$, and (ii)
$\nabla_\bth V_F(\bth,\bl(\bth)) = \bz$, for all $\bth \in \R^d$.
\een
For brevity, we denote the Lipschitz constants $L_{V_F}$ and $L_{V_S}$
by $L_F$ and $L_S$ respectively

The assumptions imply the following:

\bit
\item The slow subsystem $\dot{\bth} = \f(\bth,\bl(\bth))$ is
\textbf{globally exponentially stable} in the following sense:
There exist constants $\g_S$ and $\kappa_S$ such that,
for every $\bth_0 \in \R^d$, we have that
\be\label{eq:328}
\nmeu{\bth(t)} \leq \g_S \nmeu{\bth_0} \exp( - \kappa_S t ), \fa t \geq 0.
\ee
\item The fast subsystem $\dot{\bphi}$ (note that there is no $\e$)
is \textbf{uniformly globally exponentially stable} with respect to $\bth$,
in the following sense:
There exist constants $\g_F$ and $\kappa_F$ such that,
for every $\bth \in \R^d$ and every $\bphi_0 \in \R^l$, we have that
\be\label{eq:329}
\nmeu{\bphi(t) - \bl(\bth)} \leq \g_F \nmeu{\bphi_0 - \bl(\bth)}
\exp( - \kappa_F t ), \fa t \geq 0.
\ee
\eit
Conversely, \eqref{eq:328} implies (VS1), and \eqref{eq:329} implies (VF1).
However, they do not imply (VS2) and (VF2).
Instead of \textit{postulating} the existence of suitable Lyapunov functions,
one can mimic the contents of \cite{MV-MCSS23}, and \textit{deduce} their
existence by imposing conditions on the vector fields $\f$ and $\gbold$.
We leave that to the reader.

\subsection{Stability of the Singularly Perturbed System}\label{ssec:Stab}

In this subsection, it is shown that the assumptions in Section
\ref{ssec:Ass} imply the global exponential stability of the singularly
perturbed system \eqref{eq:17} for \textit{sufficiently small}
values of $\e$.
The contents of this subsection mimic \cite{Saberi-Khalil-TAC84}, and
form the starting point for the analysis of the TTSSA algorithm of
\eqref{eq:12}.

\begin{lemma}\label{lemma:31}
Suppose the assumptions in Section \ref{ssec:Ass} hold.
Then there exist constants $D_1, D_2, D_3$ such that
\be\label{eq:331}
\IP{\nabla_\bth V_F(\bth,\bphi)}{\f(\bth,\bphi)}
\leq D_1 \nmeusq{\bphi - \bl(\bth)}
+ D_2 \nmeu{\bth} \cdot \nmeu{ \bphi - \bl(\bth) } ,
\ee
\be\label{eq:332}
\IP{\nabla V_S(\bth)}{\f(\bth,\bphi) - \f(\bth, \bl(\bth)}
\leq D_3 \nmeu{\bth} \cdot \nmeu{ \bphi - \bl(\bth) } .
\ee
\end{lemma}

\textbf{Remark:} 
The inequalities \eqref{eq:331} and \eqref{eq:332}
are \textit{assumed to hold} in \cite{Saberi-Khalil-TAC84}.
However, it is shown here that these need not be assumed -- they are a
consequence of the assumptions already made.

\begin{proof}
Recall that $\nabla_\bth V_F( \bth,\bl(\bth)) = \bz$.
Therefore
\bd
\nmeu{ \nabla_\bth V_F(\bth,\bphi) } 
= \nmeu{ \nabla_\bth V_F(\bth,\bphi) - \nabla_\bth V_F( \bth,\bl(\bth)) }
\leq L_F \nmeu{ \bphi - \bl(\bth) } .
\ed
Next,
\bd
\f(\bth,\bphi) = \f(\bth,\bl(\bth)) + [ \f(\bth,\bphi) - \f(\bth,\bl(\bth)) ] .
\ed
Hence
\bd
\nmeu{\f(\bth,\bphi)} \leq \nmeu{ \f(\bth,\bl(\bth)) } 
+ \nmeu{  \f(\bth,\bphi) - \f(\bth,\bl(\bth)) } .
\ed
Next, note that, for $\ubold \in \R^d , \vbold \in \R^l$, we have that
\bd
\nmeu{ (\ubold, \vbold) } \leq \nmeu{ \ubold} + \nmeu{\vbold}.
\ed
Consequently
\bd
\nmeu{ \f(\bth,\bl(\bth)) } \leq L_\f ( \nmeu{\bth} + \nmeu{\bl(\bth)} )
\leq L_\f ( 1 + L_\bl ) \nmeu{\bth} .
\ed
Next
\bd
\nmeu{  \f(\bth,\bphi) - \f(\bth,\bl(\bth)) } \leq
L_\f \nmeu{\bphi - \bl(\bth)} ,
\ed
Combining these bounds shows that 
\bd
\nmeu{\f(\bth,\bphi)} \leq L_\f ( 1 + L_\bl ) \nmeu{\bth} 
+ L_\f \nmeu{\bphi - \bl(\bth)} .
\ed
Thus we have bounds for the norms of both the terms in the inner product
in \eqref{eq:321}.
Now applying Schwarz' inequality gives \eqref{eq:321} with
\bd
D_1 = L_F L_\f , D_2 = L_F L_\f (1 + L_\bl) .
\ed

To derive \eqref{eq:321}, observe that
\bd
\nmeu{\nabla V(\bth)} \leq L_S \nmeu{\bth} .
\ed
Also, as shown above
\bd
\nmeu{  \f(\bth,\bphi) - \f(\bth,\bl(\bth)) } \leq
L_\f \nmeu{\bphi - \bl(\bth)} .
\ed
Combining all these bounds and applying Schwarz' inequality leads to
\eqref{eq:332} with 
\bd
D_3 = L_S L_\f ( 1 + L_\bl ) .
\ed
This completes the proof.
\end{proof}

With this preliminary result out of the way, we follow
\cite{Saberi-Khalil-TAC84} to construct a Lyapunov function
for the singularly perturbed system \eqref{eq:17} in the form
\be\label{eq:333}
V_d(\bth,\bphi) := (1-d) V_S(\bth) + d V_F(\bth,\bphi) ,
\ee
where $d \in (0,1)$ is a constant chosen by the user.

\begin{theorem}\label{thm:33}
Under the stated assumptions, we have that
\be\label{eq:334}
\dot{V}_d(\bth,\bphi) \leq - \ubold^\top M(\e) \ubold ,
\ee
where
\be\label{eq:335}
\ubold := \left[ \ba{c} \nmeu{\bth} \\ \nmeu{\bphi - \bl(\bth)} \ea \right] ,
\ee
\be\label{eq:336}
M(\e) := \left[ \ba{ll} (1-d) C_S & -d D_3 \\ -(1-d) D_2 & (d C_F D_1)/\e
\ea \right] .
\ee
Moreover, the matrix $[ M(\e) + M^\top(\e) ]/2$ is positive definite,
and the system \eqref{eq:17} is globally exponentially stable, whenever
\be\label{eq:337}
\e < \e^*(d) := \frac{4d(1-d) C_S C_F D_1}{[(1-d)D_2 + d D_3]^2} .
\ee
\end{theorem}

\begin{proof}
The proof consists of computing the time derivative $\dot{V}_d$, via
\bd
\dot{V}_d = (1-d) \IP{ \nabla V_S(\bth)} { \f(\bth,\bphi)}
+ d \IP{ \nabla_\bth V_F(\bth,\bphi)} { \f(\bth(,\bphi)} 
+ \frac{d}{\e} \IP{ \nabla_\bphi V_F(\bth,\bphi)} { \gbold(\bth(,\bphi)} .
\ed
We will analyze each of the three terms on the right side (without the
constants $1-d$ and $d$), using Lemma \ref{lemma:31}.
First
\begin{eqnarray*}
\IP{ \nabla V_S(\bth)} { \f(\bth,\bphi)} & = &
\IP{ \nabla V_S(\bth)} { \f(\bth,\bl(\bth))}
+ \IP{ \nabla V_S(\bth)} {\f(\bth,\bphi) - \f(\bth,\bl(\bth))} \\
& \leq & - C_S \nmeusq{\bth} 
+ D_3 \nmeu{\bth} \cdot \nmeu{ \bphi - \bl(\bth) } .
\end{eqnarray*}
Second
\bd
\IP{\nabla_\bth V_F(\bth,\bphi)}{\f(\bth,\bphi)}
\leq D_1 \nmeusq{\bphi - \bl(\bth)}
+ D_2 \nmeu{\bth} \cdot \nmeu{ \bphi - \bl(\bth) } .
\ed
Third
\bd
\IP{ \nabla_\bphi V_F(\bth,\bphi)} { \gbold(\bth(,\bphi)}
\leq - C_F \nmeusq{\bphi - \bl(\bth)} .
\ed
Combining these three bounds, and using the notation $\ubold$ defined in
\eqref{eq:335} gives \eqref{eq:334} with $M(\e)$ defined as in \eqref{eq:336}.
For the final step, note that
\bd
[ M(\e) + M^\top(\e) ]/2 =
\left[ \ba{cc}
 (1-d) C_S & -[ (1-d) D_2 + d D_3 ]/2 \\
-[ (1-d) D_2 + d D_3 ]/2  & (d C_F D_1)/\e
\ea \right] .
\ed
The trace of this matrix is always positive, while the determinant is
positive whenever \eqref{eq:337} holds.
\end{proof}

Note that if we actually know the various constants in $M(\e)$, we can
try to find an ``optimal'' choice of $d$ so as to maximize $\e^*(d)$.
However, in TTSSA, the fixed constant $\e$ is replaced by a 
\textit{sequence} $\{ \e_t \}$ that converges to zero.
Therefore there is no gain in trying to optimize the choice of $d$,
and we might as well takd $d = 0.5$.
However, we do not do this, just in case some future researcher may
wish to use a different choice of $d$.

\section{Main Results}\label{sec:Main}

With the preliminary results in place, we are in a position to state
and prove the main results of the paper.

\subsection{Main Theorems}\label{ssec:Thms}

In order to study the TTSSA algorithm of \eqref{eq:12}, it is assumed that
the dynamical system \eqref{eq:17} satisfies the assumptions of
Theorem \ref{thm:33}.
In addition, some assumptions are made regarding the measurement
errors.
As defined earlier, let $\F_t$ denote the $s$-algebra generated by
$\bth_0, \bphi_0$, and the measurement sequences
$\y_1^t := (\y_1 , \cdots , \y_t)$, $\z_1^t := (\z_1 , \cdots , \z_t)$.
Recall that $E_t(X)$ denotes the conditional expectation $E(X|F_t)$.
Also recall the following definitions from \eqref{eq:12a} and \eqref{eq:12b}:
\bd
\x_t := E_t(\y_{t+1} - \f(\bth_t, \bphi_t)) , \quad
\bxt := \y_{t+1} - \f(\bth_t, \bphi_t) - \x_t ,
\ed
\bd
\w_t := E_t(\z_{t+1} - \gbold(\bth_t, \bphi_t)) , \quad
\bzt := \z_{t+1} - \gbold(\bth_t, \bphi_t) - \w_t .
\ed

With these conventions, the following assumptions are made:
\ben
\item[(N1)] There exist constants $B_{t,S}$, $B_{t,F}$ such that
\be\label{eq:411}
\nmeu{\x_t} \leq B_{t,S} ( 1 + \nmeu{\ubold_t}) , \quad
\nmeu{\w_t} \leq B_{t,F} ( 1 + \nmeu{\ubold_t}) ,
\ee
where, in analogy with \eqref{eq:335}, we define
\be\label{eq:412}
\ubold_t := \left[ \ba{c} \nmeu{\bth_t} \\ \nmeu{\bphi_t - \bl(\bth_t)} \ea
\right] .
\ee
\item[(N2)] There exist constants $M_{t,S}$, $M_{t,F}$ such that
\be\label{eq:413}
E_t( \nmeusq{\bxt} ) \leq M_{t,S}^2 ( 1 + \nmeusq{\ubold_t} ) , \quad
E_t( \nmeusq{\bzt} ) \leq M_{t,F}^2 ( 1 + \nmeusq{\ubold_t} ) .
\ee
\een

\textbf{Remarks:}
\bit
\item Assumption (N1) permits $\y_{t+1}$ to be a \textit{biased} measurement
of $\f(\bth_t,\bphi_t)$, and similarly for $\z_{t+1}$.
Equation \eqref{eq:411} gives a bound on the extent of these biases.
As we shall see below, the bounds on the biases need to approach zero
as $\tai$.
\item Assumption (N2) imposes bounds on the \textit{conditional variances}
of the measurement errors.
As we shall see below, we permit the bounds $M_{t,S}$ and $M_{t,F}$
to \textit{grow without bound} as a function of $t$.
This situation occurs naturally when using Stochastic Gradient Descent
with ``zeroth-order'' estimates of the gradient, and was recognized as
far back as \cite{Kief-Wolf-AOMS52}.
This situation might not be natural in the context of TTSSA, but we
study the more general setting anyway, as it is not any more challenging
than the case where both bounds are uniformly bounded as functions of $t$.
The only price to be paid is more messay computation.
\eit

Now we state the main theorems of the paper.
Theorem \ref{thm:41} gives sufficient conditions for the convergence of
the TTSSA algorithm for the general case covered by Assumptions (N1) and (N2).
Technically, Theorem \ref{thm:41a} is a corollary of Theorem \ref{thm:41},
for the case where the two bias terms $B_{t,S}$, $B_{t,F}$ are zero,
and the two variance bounds $M_{t,S}$, $M_{t,F}$ are uniformly bounded.
However, since \textit{this} is the situation studied in almost all of the
literature on TTSSA, we state it here as a theorem, for ready reference.
Next, Theorem \ref{thm:42} gives bounds on the \textit{rate of 
convergence} of the TTSSA algorithm in the general case where (N1) and (N2)
are assumed.
In particular,
for the case where the two bias terms $B_{t,S}$, $B_{t,F}$ are zero,
and the two variance bounds $M_{t,S}$, $M_{t,F}$ are uniformly bounded.
the rate of convergence of TTSSA is shown to be $o(t^{-\eta})$
for all $\eta \in (0,1)$.
Note that this is the best bound thus far.

\begin{theorem}\label{thm:41}
Consider the TTSSA algorithm described in \eqref{eq:12}.
Suppose the assumptions in Section \ref{ssec:Ass} regarding the stability of
the system \eqref{eq:17} hold, namely (A1), (A2), (VF1), (VF2), (VS1) and (VS2).
Suppose further that the assumptions (N1) and (N2) on the measurement errors
also hold.
With these assumptions, we can state the following conclusions:
\ben
\item Suppose that
\be\label{eq:414}
\sum_{t=0}^\infty \al_t^2 < \infty, \quad 
\sum_{t=0}^\infty \al_t B_{t,S} < \infty , \quad
\sum_{t=0}^\infty \al_t^2 M_{t,S}^2 < \infty ,
\ee
and 
\be\label{eq:414a}
\sum_{t=0}^\infty \beta_t^2 < \infty , \quad
\sum_{t=0}^\infty \beta_t B_{t,F} < \infty , \quad
\sum_{t=0}^\infty \beta_t^2 M_{t,F}^2 < \infty 
\ee
Then
\bit
\item The stochastic processes $V_S(\bth_t)$ and $V_F(\bth_t,\bphi_t)$
are both bounded almost surely.
\item There exists a random variable $X$ such that
\be\label{eq:415}
V_d(\bth_t,\bphi_t) = (1-d) V_S(\bth_t) + d V_F(\bth_t,\bphi_t) \ap X \as ,
\mbox{ as } \tai .
\ee
\eit
\item Suppose in addition that
\be\label{eq:416}
\sum_{t=0}^\infty \al_t = \infty , \quad
\frac{\al_t}{\beta_t} \ap 0 \mbox{ as } \tai .
\ee
Then
$V_S(\bth_t)$ and $V_F(\bth_t, \bphi_t)$ both approach zero almost
surely as $\tai$.
Further $\bth_t \ap \bz$ as $\tai$, and
\be\label{eq:417}
\sum_{t=0}^\infty \beta_t \nmeusq{\bphi_t - \bl(\bth_t)} < \infty .
\ee
\een
\end{theorem}

\textbf{Remarks}
\bit
\item Equations \eqref{eq:414}, \eqref{eq:414a}, and \eqref{eq:416}
are the two time-scale
analogs of the Kiefer-Wolfowitz-Blum conditions \cite{Kief-Wolf-AOMS52,Blum54}.
\item Combining the first bullet of Conclusion 1 with
\eqref{eq:324} and \eqref{eq:326} leads to the conclusion that
$\bth_t$ and $\bphi_t - \bl(\bth_t)$ are both bounded
almost surely.
Since $\bl(\cdot)$ is $GLC$, this leads to the further conclusion
that both $\bth_t$ and $\bphi_t$ are bounded almost surely.
In the ODE approach, this property s referred to as ``stability,''
and is either assumed, or established using other arguments.
In the martingale approach, the stability of the iterations is a ready
consequence of the assumptions in \eqref{eq:414} and \eqref{eq:414a}.
Note that the boundedness of the iterations does not require \eqref{eq:416}.
\item Coming now to Conclusion 2, \eqref{eq:416} implies that
\be\label{eq:418}
\sum_{t=0}^\infty \beta_t = \infty .
\ee
\item Once it is established that 
$V_S(\bth_t)$ and $V_S(\bth_t, \bphi_t)$ both approach zero almost
surely as $\tai$, it readily follows from \eqref{eq:324} and \eqref{eq:326}
that $\bth_t - \bl(\bth_t) \ap \bz$ as $\tai$, and that $\bth_t \ap \bz$
as $\tai$.
Moreover, it is established in the proof that
\bd
\sum_{t=0}^\infty \al_t \nmeusq{\bth_t} < \infty .
\ed
Thus \eqref{eq:417} contains the intuitive
belief that $\bphi_t - \bl(\bth_t)$ converges to zero ``more quickly''
than $\bth_t$.
\eit

Now we state the corollary of Theorem \ref{thm:41} for the case of
unbiased measurements with bounded conditional variance.

\begin{theorem}\label{thm:41a}
Consider the TTSSA algorithm described in \eqref{eq:12}.
Suppose the assumptions in Section \ref{ssec:Ass} regarding the stability of
the system \eqref{eq:17} hold, namely (A1), (A2), (VF1), (VF2), (VS1) and (VS2).
Suppose further that the assumptions (N1) and (N2) hold, with
the simplified assumptions
\be\label{eq:414b}
B_{t,S} = 0 , B_{t,F} = 0 \fa t , \quad
\sup_t M_{t,S} < \infty , \quad \sup_t M_{t,F} < \infty .
\ee
Under these conditions:
\ben
\item Suppose that
\be\label{eq:414c}
\sum_{t=0}^\infty \al_t^2 < \infty, \quad 
\sum_{t=0}^\infty \beta_t^2 < \infty .
\ee
Then the stochastic processes $V_S(\bth_t)$ and $V_F(\bth_t,\bphi_t)$
are both bounded almost surely,
and there exists a random variable $X$ such that \eqref{eq:415} holds.
\item Suppose in addition that \eqref{eq:416} holds, i.e.,
\be\label{eq:414d}
\sum_{t=0}^\infty \al_t = \infty , \quad
\frac{\al_t}{\beta_t} \ap 0 \mbox{ as } \tai .
\ee
Then $V_S(\bth_t)$ and $V_F(\bth_t, \bphi_t)$ both approach zero almost
surely as $\tai$.
Further $\bth_t \ap \bz$ as $\tai$, and \eqref{eq:417} holds.
\een
\end{theorem}

\textbf{Remarks:}
Note that, as before, the hypotheses imply that
\bd
\sum_{t=0}^\infty \beta_t = \infty .
\ed
Hence \eqref{eq:414c} and \eqref{eq:414d} are the two time-scale
analogs of the well-known Robbins-Monro conditions \cite{Robbins-Monro51}.

Until now we have analyzed only the convergence of the TTSSA algorithm.
Now we analyze the \textit{rate} of convergence using Theorem
\ref{thm:33}.
Specifically, we wish to find a bound on $\eta$ such that
\be\label{eq:4112}
V_d(\bth_t,\bphi_t) , \nmeusq{\bth_t} , \nmeusq{\bphi_t - \bl(\bth_t)} 
= o(t^{-\eta}) .
\ee
For this purpose, we assume the following:
There exist nonnegative constants $\g_S$, $\g_F$,
and $\nu_S \in [0,0.5)$, $\nu_F \in [0,0.5)$ such that
\be\label{eq:419a}
B_{t,S} = O(t^{-\g_S}), \quad B_{t,F} = O(t^{-\g_F}) , \quad
M_{t,S} = O(t^{\nu_S}), \quad M_{t,F} = O(t^{\nu_F}) .
\ee
If $B_{t,S} = 0$ and $B_{t,F} = 0$ for all $t$, we take $\g_S = \g_F = 1$.
Note that the ``biases'' $B_{t,S}$ and $B_{t,F}$
are required to decrease with $t$, while the variances $M_{t,S}$ and $M_{t,F}$
are permitted to increase with $t$.
Moreover, it is assumed that $\nu_S < 0.5$ and $\nu_F < 0.5$, because otherwise
the conditions of the theorem are impossible to satisfy.

\begin{theorem}\label{thm:42}
Suppose all the hypotheses of Theorem \ref{thm:41} hold, and in
addition, \eqref{eq:419a} also holds.
Choose the step size sequences as
\be\label{eq:419b}
\al_t = O(t^{-1}) , \quad \al_t = \OM(t^{-1}) , \quad
\beta_t = O(t^{-(1-\D)}) , \beta_t = \OM(t^{-(1-\D)}) ,
\ee
where $\D$ is arbitrarily small.
With this choice of step size sequences, we have the following.
\ben
\item
In the general case, \eqref{eq:4112} holds with
\be\label{eq:4112a}
\eta < \min \{ \g_S , \g_F  , 1 - 2 \max \{ \nu_S, \nu_F \} \} .
\ee
\item
If $B_{t,S}$ and $B_{t,F}$ are identically zero, and
\bd
\sup_t M_{t,S} < \infty , \quad \sup_t M_{t,F} < \infty ,
\ed
then \eqref{eq:4112} holds with $\eta < 1$.
\een
\end{theorem}

\textbf{Remarks:}
\bit
\item It is easily verified that the various order
bounds in \eqref{eq:419b} guarantee
that \eqref{eq:414c} and \eqref{eq:414d} are satisfied.
\item
As mentioned in Section \ref{ssec:Contrib}, this is apparently the first time
that a convergence rate of $o(t^{-\eta})$ for $\eta$ arbitrarily close
to one has been established.
\eit

\subsection{Proofs of Main Theorems}\label{ssec:Proofs}

\subsubsection{Proof of Theorem \ref{thm:41}}

The proof of Theorem \ref{thm:41} is based on Theorem \ref{thm:Robb-Sieg}.
We use the Lyapunov function $V_d$ of \eqref{eq:333}, and show
that it satisfies \eqref{eq:311} with $z_t = V_d(\bth_t,\bphi_t)$,
for suitable choices of the terms $f_t, g_t$ and $h_t$, and $\al_t$
as above.
This is accomplished by obtaining suitable upper bounds for
$V_d(\bth_{t+1},\bphi_{t+1})$.
Because
\bd
V_d(\bth_{t+1},\bphi_{t+1}) = (1-d) V_S(\bth_{t+1})
+ d V_F(\bth_{t+1},\bphi_{t+1}) ,
\ed
we can find bounds for each term separately and combine the bounds.

Finding these upper bounds is facilitated by the following fact,
which is stated as \cite[Eq.\ (2.4)]{Ber-Tsi-SIAM00}.
It states the following:
Suppose $J: \R^d \ap \R$ is $\C^1$, and that $\gJ(\cdot)$
is $GLC$ with Lipschitz constant $L$.
Then
\be\label{eq:421}
J(\x + \y) \leq J(\x) + \IP{\gJ(\x)}{\y} + \frac{L}{2} \nmeusq{\y} .
\ee
Note that the above bound is also the upper bound given in
\cite[Eq.\ 2.1.9]{Nesterov18}.
The other part (the lower bound) implies that, if the above bound
holds with the inequality reversed, then $J(\cdot)$ is convex.
However, no such assumption is made here.

Now we begin by finding an upper bound for $V_S(\bth_{t+1})$.
Using \eqref{eq:421} gives
\begin{eqnarray*}
V_S(\bth_{t+1}) & = & V_S(\bth_t + \al_t \y_{t+1}) \\
& \leq & V_S(\bth_t) + \al_t \IP{\nabla V_S(\bth_t)}{\y_{t+1}}
+ \frac{\al_t^2}{2} L_S \nmeusq{\y_{t+1}} .
\end{eqnarray*}
Applying the operation $E_t(\cdot)$ to both sides, invoking \eqref{eq:12a},
and using Assumptions (N1) gives
\be
\begin{split}
E_t(V_S(\bth_{t+1})) & \leq V_S(\bth_t) 
+ \al_t \IP{\nabla V_S(\bth_t)}{\f(\bth_t)} 
+ \al_t \IP{\nabla V_S(\bth_t)}{\x_t} \\
& +  \frac{\al_t^2}{2} L_S E_t (\nmeusq{\x_t}) 
+ \frac{\al_t^2}{2} L_S E_t (\nmeusq{\bxt}) .
\end{split}
\label{eq:422}
\ee
Let us analyze the third, fourth and fifth terms on the right side,
while making liberal use of the obvious inequality $2x \leq 1+ x^2$
to replace linear terms by constant plus quadratic terms.
We also make use of the Lipschitz-continuity of $\nabla V_S$
with constant $L_S$.
This leads to
\begin{eqnarray}
\IP{\nabla V_S(\bth_t)}{\x_t} & \leq & \nmeu{\nabla V_S(\bth_t)}
\cdot \nmeu{\x_t} \nonumber \\
& \leq & L_S \nmeu{\bth_t} \cdot B_{t,S} (1 + \nmeu{\ubold_t}) \nonumber \\
& \leq & L_S B_{t,S} ( \nmeu{\ubold_t} + \nmeusq{\ubold_t} ) \nonumber \\
& \leq & L_S B_{t,S} ( 0.5 \nmeu{\ubold_t} + 1.5 \nmeusq{\ubold_t} ) \nonumber \\
& \leq & L_S B_{t,S} ( \nmeu{\ubold_t} + 2 \nmeusq{\ubold_t} ) .
\label{eq:423}
\end{eqnarray}
where, in analogy with \eqref{eq:335},
\bd
\ubold_t := \left[ \ba{c} \nmeu{\bth_t} \\ \nmeu{\bphi_t - \bl(\bth_t)}
\ea \right] ,
\ed
Next, 
\be\label{eq:424}
\nmeusq{\x_t} \leq B_{t,S}^2 ( 1 + \nmeu{\ubold_t} )^2
\leq B_{t,S}^2 (1 + 2 \nmeu{\ubold_t} + \nmeusq{\ubold_t} )
\leq 2 B_{t,S}^2 ( 1 + \nmeusq{\ubold_t} ) .
\ee
Of course, \eqref{eq:413} gives a bound for $E_t(\nmeusq{\bxt}$.
Substituting all of these bounds into \eqref{eq:422} gives
\be\label{eq:425}
E_t(V_S(\bth_{t+1})) \leq V_S(\bth_t)
+ \al_t \IP{\nabla V_S(\bth_t)}{\f(\bth_t)}+ R_{t,S},
\ee
where $R_{S,t}$ denotes the ``slow residual''
\be\label{eq:426}
R_{t,S} := \al_t B_{t,S} L_S ( 1 + 2 \nmeusq{\ubold_t} )
+ \al_t^2 B_{t,S}^2 L_S ( 1 + \nmeusq{\ubold_t} )
+ \al_t^2 M_{t,S}^2 \frac{L_S}{2} ( 1 + \nmeusq{\ubold_t} ) .
\ee

A similar analysis can be carried out on $V_F(\bth_{t+1},\bphi_{t+1})$.
This gives
\begin{eqnarray*}
E_t(V_F(\bth_{t+1},\bphi_{t+1})) & = &
E_t( V_F(\bth_t + \al_t \y_{t+1}, \bphi_t + \beta_t \z_{t+1}) ) \\
& \leq & V_F(\bth_t,\bphi_t)
+ \al_t \IP{\nabla_\bth V_F(\bth_t,\bphi_t)}{\f(\bth_t,\bphi_t)} \\
& + & \beta_t \IP{\nabla_\bphi V_F(\bth_t,\bphi_t)}{\gbold(\bth_t,\bphi_t)} \\
& + & \al_t \IP{\nabla_\bth V_F(\bth_t,\bphi_t)}{\x_t}
+ \beta_t \IP{\nabla_\bphi V_F(\bth_t,\bphi_t)}{\w_t} \\
& + & \frac{\al_t^2}{2} L_F (\nmeusq{\x_t})
+ \frac{\al_t^2}{2} L_F E_t (\nmeusq{\bxt}) \\ 
& + & \frac{\beta_t^2}{2} L_F \nmeusq{\w_t} 
+ \frac{\beta_t^2}{2} L_F E_t (\nmeusq{\bzt}) .
\end{eqnarray*}
The trailing terms can be bounded in a manner analogous to $V_S(\bth_{t+1})$,
and gives
\be\label{eq:427}
\begin{split}
E_t(V_F(\bth_{t+1},\bphi_{t+1})) & \leq
V_F(\bth_t,\bphi_t)
+ \al_t \IP{\nabla_\bth V_F(\bth_t,\bphi_t)}{\f(\bth_t,\bphi_t)} \\
& + \beta_t \IP{\nabla_\bphi V_F(\bth_t,\bphi_t)}{\gbold(\bth_t,\bphi_t)} 
+ R_{t,F} ,
\end{split}
\ee
where $R_{t,F}$ denotes the ``fast residual'' defined by
\be\label{eq:428}
\begin{split}
R_{t,F} & := \al_t B_{t,S} L_F ( 1 + 2 \nmeusq{\ubold_t} )
+ \al_t^2 B_{t,S}^2 L_F ( 1 + \nmeusq{\ubold_t} )
+ \al_t^2 M_{t,S}^2 \frac{L_F}{2} ( 1 + \nmeusq{\ubold_t} ) \\
& + \beta_t B_{t,F} L_F ( 1 + 2 \nmeusq{\ubold_t} )
+ \beta_t^2 B_{t,F}^2 L_F ( 1 + \nmeusq{\ubold_t} )
+ \beta_t^2 M_{t,F}^2 \frac{L_F}{2} ( 1 + \nmeusq{\ubold_t} ) .
\end{split}
\ee

Now we can substitute the above bounds into the definition of $V_d$,
and collect terms.
This gives
\beq
E_t(V_d(\bth_{t+1},\bphi_{t+1})) & = &
(1-d) E_t(V_S(\bth_{t+1})) + d E_t(V_F(\bth_{t+1},\bphi_{t+1})) \nonumber \\
& \leq & V_d(\bth_t,\bphi_t) 
+ \al_t ( 1-d) \IP{\nabla_\bth V_F(\bth_t,\bphi_t)}{\f(\bth_t,\bphi_t)}
\nonumber \\
& + & \beta_t d \IP{\nabla_\bphi V_F(\bth_t,\bphi_t)}{\gbold(\bth_t,\bphi_t)} 
+ (1-d) R_{t,S} + d R_{t,F} .
\label{eq:429}
\eeq
Now we use the analysis in Section \ref{ssec:Stab}, specifically
\eqref{eq:334}, and observe that
\bd
\al_t ( 1-d) \IP{\nabla_\bth V_F(\bth_t,\bphi_t)}{\f(\bth_t,\bphi_t)}
+ \beta_t d \IP{\nabla_\bphi V_F(\bth_t,\bphi_t)}{\gbold(\bth_t,\bphi_t)}
= \dot{V}_d (\bth_t,\bphi_t) \leq \ubold_t^\top M(\e(t)) \ubold_t,
\ed
with $\e(t) = \al(t)/\beta(t)$.
This leads to
\be\label{eq:4210}
E_t(V_d(\bth_{t+1},\bphi_{t+1})) \leq V_d(\bth_t,\bphi_t)
- \al_t \ubold_t^\top M(\e(t)) \ubold_t + R_t ,
\ee
where $R_t$ denotes the ``combined'' residual term
\bd
R_t = (1-d) R_{t,S} + d R_{t,F} .
\ed

The next step in the proof is to show that the ``combined residual''
satisfies an upper bound  of the form
\be\label{eq:4211}
R_t \leq f_t V_d(\bth_t,\bphi_t) + g_t 
\ee
for some summable sequences $\{ f_t \}$ and
$\{ g_t \}$.\footnote{$f_t$ and $g_t$ should
not be confused with $\f$ and $\gbold$.}
Once this is done, we can apply Theorem \ref{thm:Robb-Sieg} to
\eqref{eq:4210} with
\bd
z_t = V_d(\bth_t,\bphi_t) .
\ed
Towards this end, note that, ignoring constants, the residual $R_t$
is a sum of multiples of either $1$ or of $\nmeusq{\ubold_t}$,
multiplied by one of eight terms:
\bd
\al_t B_{t,S} , \al_t^2 B_{t,S}^2 , \al_t^2 , \al_t^2 M_{t,S}^2 ,
\ed
and
\bd
\beta_t B_{t,F} , \beta_t^2 B_{t,F}^2 , \beta_t^2 , \beta_t^2 M_{t,F}^2 ,
\ed
From \eqref{eq:414} and \eqref{eq:414a}, we know that all sequences are
summable, except $\al_t^2 B_{t,S}^2$ and $\beta_t^2 B_{t,F}^2$.
However, the summability of these sequences follows from the summability
of the sequences $\al_t B_{t,S}$ and $\beta_t B_{t,F}$ respectively.
(Note that $\ell_1 \seq \ell_2$.)
Moreover, (VS1) and (VF1) together show that there exist constants
$C_{\min}$ and $C_{max}$ such that
\bd
C_{min} \nmeusq{\ubold_t} \leq V_d(\bth_t,\bphi_t) \leq
C_{max} \nmeusq{\ubold_t} .
\ed
Combining all these observations establishes \eqref{eq:4211}.

Now we can apply Theorem \ref{thm:Robb-Sieg}, with
$z_t := V_d(\bth_t,\bphi_t)$.
This leads to two observations:
First, $V_d(\bth_t,\bphi_t)$ is bounded
almost surely, and that it converges almost surely to some random variable $X$.
This is the first conclusion of the theorem.
The second observation is that
\be\label{eq:4212}
\sum_{t=0}^\infty \al_t \ubold_t^\top M(\e(t)) \ubold_t < \infty
\ee
almost surely.
To prove the second conclusion of the theorem using \eqref{eq:4212},
we proceed as follows.
Recall from \eqref{eq:336} that
\bd
M(\e(t)) := \left[ \ba{ll} (1-d) C_S & -d D_3 \\ -(1-d) D_2 & (d C_F D_1)/\e(t)
\ea \right] .
\ed
Now we partition $M(\e(t))$ as
\bd
M(\e(t)) = M_1(\e(t)) + M_2(\e(t)) ,
\ed
where
\bd
M_1(\e(t)) = \left[ \ba{ll} (1-d) C_S & -d D_3 \\ -(1-d) D_2 &
(d C_F D_1)/(2\e(t)) \ea \right] ,
M_2(t) =  \left[ \ba{cc} 0 & 0  \\ 0 & d C_F D_1/(2 \e(t)) \ea \right] .
\ed
Note that $M_2(\e(t))$ is positive semidefinite whenever $\e(t) > 0$
(which it always is).
Moreover, from Theorem \ref{thm:33}, we know that $M_1(\e(t))$
is positive definite whenever $\e(t) < \e^*(d)/2$, say for
$\e(t) \leq \e^*(d)/3$.
Since one part of \eqref{eq:416} states that $\e(t) = \al_t/\beta_t \ap 0$
as $\tai$, we choose a finite time $T$ such that $\e(t) \leq \e^*(d)/3$
for all $t \geq T$, or $2 \e(t) < (2\e^*(d))/3$.
Define
\bd
\bar{M} = \left[ \ba{ll} (1-d) C_S & -d D_3 \\ -(1-d) D_2 &
(d C_F D_1)/((2\e^*(d))/3) \ea \right] ,
\ed
and observe that $\bar{M}$ is independent of $t$.
Moreover
\bd
M_1(\e(t)) - \bar{M} =
\left[ \ba{cc} 0 & 0  \\ 0 & d C_F D_1 \left( \frac{1}{2 \e(t)}
- \frac{1}{2\e^*(d))/3} \right) \ea \right] 
\ed
is positive semidefinite whenever $t \geq T$.
Therefore
\bd
M_1(\e(t)) \geq \bar{M} , \quad
M(\e(t)) \geq \bar{M} + M_2(\e(t)) , \fa t \geq T ,
\ed
where $A \geq B$ means that $A-B$ is positive semidefinite.
Now \eqref{eq:4212} implies that
\be\label{eq:4213}
\sum_{t=T}^\infty \al_t \ubold_t^\top M(\e(t)) \ubold_t < \infty ,
\ee
which in turn implies that
\bd
\sum_{t=T}^\infty \al_t \ubold_t^\top [ \bar{M} + M_2 
(\e(t))] \ubold_t < \infty .
\ed
Next, observe that if a sum of two nonnegative quantities is finite,
then each of the quantities is individually finite.
Therefore \eqref{eq:427} implies that
\be\label{eq:4214}
\sum_{t=T}^\infty \al_t \ubold_t^\top \bar{M} \ubold_t < \infty ,
\ee
as well as
\bd
\sum_{t=T}^\infty \al_t \ubold_t^\top M_2(\e(t)) \ubold_t < \infty .
\ed
However, $\al_t / \e(t) = \beta_t$, so that
\bd
\al_t \ubold_t^\top M_2(\e(t)) \ubold_t =
\frac{d C_F D_1}{2} \beta_t \nmeusq{\bphi_t - \bl(\bth_t)} .
\ed
Therefore
\be\label{eq:4215}
\frac{d C_F D_1}{2} \sum_{t=T}^\infty \beta_t \nmeusq{\bphi_t - \bl(\bth_t)}
< \infty.
\ee
Note that \eqref{eq:428} is \eqref{eq:417} after discarding the constant factor.
It remains only to show that $V_F(\bth_t)$ and $V_S(\bth_t, \bphi_t)$
both approach zero almost surely as $\tai$.

For this purpose, observe that, due to \eqref{eq:324} and \eqref{eq:326},
and the positive definiteness of $\bar{M}$, there exists a constant $C_M$
such that
\bd
V_d(\bth,\bphi) \leq C_M \ubold^\top \bar{M} \ubold .
\ed
Therefore \eqref{eq:428} implies that
\be\label{eq:4216}
\sum_{t=0}^\infty \al_t V_d(\bth_t,\bphi_t) < \infty
\ee
almost surely.
To proceed, recall that $V_d(\bth_t,\bphi_t)$ converges to some
random variable $X$ as $\tai$.
Thus, if it can be shown that $X = 0$ almost surely, then it would follow
that $V_S(\bth_t)$ and $V_F(\bphi_t,\bl(\bth_t))$ each individually converge 
to zero as $\tai$.
The proof is by contradiction.
Recall that all random variables act on some probability space $(\OM,\SI,P)$.
Let $\om \in \OM$ indicate the ``sample path'' indicator,
and denote $V_d(\bth_t(\om),\bphi_t(\om))$ by just $V_d(\om)$.
Suppose that there exists an $\om \in \OM$ such that $X(\om) > 0$,
say $X(\om) = 2 \eta$.
Choose a time $T$ such that $|X(\om) - V_d(\om)| \leq \eta$ for all $t \geq T$,
Then $\V_d(\om) \geq \eta$ for all $t \geq T$.
Observe that, as a consequence of \eqref{eq:416}, we have
\bd
\sum_{t=T}^\infty \al_t = \infty .
\ed
Therefore
\bd
\sum_{t=T}^\infty \al_t V_d(\om) = \infty .
\ed
In view of \eqref{eq:4210}, we conclude that the set of $\om$ for which
$X(\om) > 0$ has to have measure zero.
This completes of Theorem \ref{thm:41}.

To prove Theorem \ref{thm:41a}, observe that if
$B_{t,S} = B_{t,F} = 0$, and the coefficients $M_{t,S}$
and $M_{t,F}$ are bounded with respect to $t$,
then \eqref{eq:414} and \eqref{eq:414a} reduce to \eqref{eq:414c}.

\subsubsection{Proof of Theorem \ref{thm:42}}

The proof of Theorem \ref{thm:42} is based on Theorem \ref{thm:32}.
With the various order assumptions in the theorem statement,
it is easy to see that the divergence condition in \eqref{eq:317} is satisfied,
so that the operative condition is the summability, namely
\bd
\sum_{t=1}^\infty (t+1)^\eta g_t(\om) < \infty .
\ed
As shown in the proof of Theorem \ref{thm:41}, $g_t$
consists of eight terms, namely:
\bd
\al_t^2 , \quad \al_t B_{t,S} , \quad \al_t^2 B_{t,S}^2 , \quad \al_t^2 M_{t,S}^2 ,
\ed
and
\bd
\beta_t^2 , \quad \beta_t B_{t,F} , \quad \beta_t^2 B_{t,F}^2 , \quad \beta_t^2 M_{t,F}^2 ,
\ed
Out of these, the terms $\al_t^2 B_{t,S}^2$ and 
$\beta_t^2 B_{t,F}^2$ can be ignored in the analysis because they decay
at least fast as $\al_t B_{t,S}$ and $\beta_t B_{t,F}$ respectively.
Also, because the exponents $\nu_S$ and $\nu_F$ are nonnegative,
the summability of $\al_t^2 M_{t,S}^2$ implies that of $\al_t^2$,
and the summability of $\beta_t^2 M_{t,F}^2$ implies that of $\beta_t^2$.
This leaves four terms, namely $\al_t B_{t,S}$, $\al_t^2 M_{t,S}^2$,
$\beta_t B_{t,F}$, and $\beta_t^2 M_{t,F}^2$.
With the order of growth / decay shown in \eqref{eq:419a} and \eqref{eq:419b},
the summability of these four sequences requires that
\bd
-1 - \g_S + \eta < -1 \imp \eta < \g_S ,
\ed
\bd
-1 + \D - \g_F + \eta < -1 \imp \eta < \g_F - \D
\imp \eta < \g_F ,
\ed
because $\D$ can be made arbitrarily small,
\bd
-2 + 2 \nu_S + \eta < -1 \imp \eta < 1 - 2 \nu_S ,
\ed
\bd
-2 + 2 \D + 2 \nu_F + \eta < -1 \imp \eta < 1 - 2 \nu_F - 2\D
\imp \eta < 1 - 2 \nu_F ,
\ed
because $\D$ can be made arbitrarily small.
Combining these four bounds gives \eqref{eq:4112a}.
Next, if $B_{t,S} \equiv 0$, we can choose $\g_S$ arbitrarilylarge , so we
can choose $\g_S = 1$.
Similarly, if $B_{t,F} \equiv 0$, we can choose $\g_F = 1$.
If $M_{t,S}$ and $M_{t,F}$ are bounded with respect to $t$,
we can choose $\eta_S = \eta_F = 0$.
Substituting these into \eqref{eq:4112a} gives $\eta < 1$ as the
bound on the rate of the convergence.

\section{Conclusions}\label{sec:Conc}

In this paper, we have analyzed
the two time-scale stochastic approximation (TTSSA) algorithm
introduced in \cite{Borkar97} using a martingale approach.
This approach leads to simple sufficient conditions for the iterations
to be bounded almost surely, as well as estimates on the rate of convergence
of the mean-squared error of the TTSSA algorithm to zero.
%
Our theory is applicable to \textit{nonlinear} equations, in
contrast to many papers in the TTSSA literature which assume that the equations
are linear.
The convergence of TTSSA is proved in the ``almost sure'' sense,
in contrast to earlier papers on TTSSA that establish convergence
in distribution, convergence in the mean, and the like.
Moreover, in this paper we establish different rates of convergence for
the fast and the slow subsystems, perhaps for the first time.
Finally, all of the above results to continue to hold in the case where
the two measurement errors have nonzero conditional
mean, and/or have conditional variances that grow without bound as the
iterations proceed.
This is in contrast to previous papers which assumed that the errors form
a martingale difference sequence with uniformly bounded conditional variance.

It is shown that when the measurement errors have zero conditional
mean and the conditional variance remains bounded, the mean-squared
error of the iterations converges to zero
at a rate of $o(t^{-\eta})$ for all $\eta \in (0,1)$.
This improves upon the rate of $O(t^{-2/3})$ proved in 
\cite{Doan-TAC23} (which is the best bound available to date).
Our bound is virtually the same as the rate of $O(t^{-1})$ proved in
\cite{Doan-arxiv24}, but for a \text{Polyak-Ruppert averaged} version of
TTSSA, and not directly.
Rates of convergence are also established for the case where the
errors have nonzero conditional mean and/or unbounded conditional variance.

Realizing the ``optimal'' convergence rate of
$o(t^{-\eta})$ for all $\eta \in (0,1)$ requires the choice
\bd
\al_t = O(t^{-1}) , \quad \beta_t = O(t^{-(1-\D)}),
\ed
where $\D$ is arbitrarily small.
With this choice, we have that $\al_t/\beta_t = O(t^{-\D})$.
Thus $\beta_t$ decays at ``almost'' the same rate as $\al_t$,
chich is somewhat counter-intuitive.
A good topic for future research is to find good rates of convergence for
TTSSA when $\beta_t$ decays much faster than $\al_t$.



\end{document}